\documentclass[11pt]{amsart}%

\usepackage{amssymb,amsmath,amsfonts,latexsym,amsthm,geometry,graphicx}
\usepackage{amsmath}
\usepackage{amsfonts}
\usepackage{amssymb}
\usepackage{graphicx}
\usepackage{color}%
\usepackage{mathrsfs,amsmath}
\usepackage{enumitem}

\usepackage[fleqn]{mathtools}

\let\savedsemi\; %this avoids tipa redefining \;, a big problem
\usepackage{tipa}
\let\;\savedsemi
\providecommand{\U}[1]{\protect\rule{.1in}{.1in}}
\providecommand{\U}[1]{\protect\rule{.1in}{.1in}}
\providecommand{\U}[1]{\protect\rule{.1in}{.1in}}
\providecommand{\U}[1]{\protect\rule{.1in}{.1in}}
\newtheorem{theorem}{Theorem}[section]

\newtheorem{proposition}[theorem]{Proposition}

\theoremstyle{definition}

\newtheorem{remark}[theorem]{Remark}
\newtheorem{definition}[theorem]{Definition}
\begin{document}	
	\sloppy %to avoid words going over the margins
	
	\title[Deriving dynamical systems for language based on the Tolerance Principle]{Deriving dynamical systems for language based on the Tolerance Principle}
	
	\author[F. C. Alves]{Fernando C. Alves}
	\thanks{Corresponding author's emails: fernando.alves@academico.ufpb.br; cabralalvesf@gmail.com}  %\indent Graduate Program in Linguistics (PROLING).\\ \indent Universidade Federal da Para\'{\i}ba \\ \indent 58.051-900 - Jo\~{a}o Pessoa, Brazil}
	

	\begin{abstract}
		 In this research note, I derive explicit dynamical systems for language within an acquisition-driven framework (Niyogi \& Berwick, 1997; Niyogi, 2006) assuming that children/learners follow the Tolerance Principle (Yang, 2016) to determine whether a rule is productive during the process of language acquisition. I consider different theoretical parameters such as population size (finite vs. infinite) and the number of previous generations that provide learners with data. Multiple simulations of the dynamics obtained here and applications to diachronic language data are in preparation, so they are not included in this first note.    
	\end{abstract}
	\maketitle

	\section{Introduction}
		
	The study of language change involves two major challenges: high complexity and difficulty to obtain systematic data. A unified account of the mechanisms of linguistic change has elluded historical linguists for centuries, from early indo-europeanists to the neogrammarians and modern linguists. Labov's seminal three-volume series on the principles of linguistic change offers a non-exhaustive study into the variety of forces driving and shaping it (1994, 2001, 2007). After decades of research, Labov argues that "The phenomenon we are studying is irrational, violent, and unpredictable. To develop principles of language change might therefore seem a quixotic undertaking, as many students of change have concluded." (1994, p.10). The methodological difficulty is self-evident. Due to the time scales involved, keeping track of specific variables over multiple generations is not usually a viable option, so systematic diachronic data are scarce and indirect methods are usually employed. These may be the main combined reasons why the first explicit dynamical systems of linguistic change were only developed in the 1990's (Niyogi \& Berwick, 1995, 1997; see Niyogi, 2006, pp.43-44), whereas other core facts about languages, such as recursion and acquisition, had been subject to mathematical analysis decades earlier (e.g. Chomsky, 1957, 1959; Gold, 1967).  
	
	Despite all difficulties, the unwavering efforts of numerous researchers from the past and the present have provided considerable empirical evidence of the qualitative nature of language dynamics, and also theoretical insights that serve as the basis for formal models. The past thirty years have seen an insurgence of computational approaches to linguistic change, combining techniques originated in a variety of disciplines, such as evolutionary biology, statistical mechanics, probability theory, computer science and linguistics (e.g. Clark \& Roberts, 1992; Niyogi \& Berwick, 1995, 1997;  Blythe \& Croft, 2012; Kiparsky, 2012; Siva, Tao, \& Marcolli, 2017; Burridge, 2017; Burridge \& Blaxter, 2021; Hayes, 2022; Ringe \& Yang, 2022). A serious survey of these various approaches and topics is far beyond our current scope and shall be carried out in another text.
	 
	Early and persisting work by generative phonologists trained in historical linguistics advanced a foundation for computational models of linguistic change (see Kiparsky, 1965, and references therein; see also the discussion on imperfect learning in Section \ref{NBsection} below). However, from that point up until the 1990's, change had received an increasingly peripherical theoretical status within the generative research program. For example, the investigation of the phenomenom is not even listed as a research problem in Chomsky's \textit{Knowledge of Language: Its Nature, Origin, and Use} (1986, p.3). As explained by Labov (2001, p.3), who set his research from the opposite end,   
	\vspace{0.5cm}
	\begin{quote}
		The search for a universal,	unchanging, and indeed unchangeable, grammar is oriented in an entirely
		different direction. As such, it lies outside the scope of this work, which is concerned with everything in language that changes or has changed. This seems to include much the larger part of linguistic categories, structures,	and substance. It therefore seems natural to ask whether we understand the forces that are responsible for the extraordinary transformations that
		affect all but a bare skeleton of abstract relations.
	\end{quote}
	\vspace{0.5cm}
	
	Chomsky's program to identify the simple defining properties of natural languages that can be plausibily accounted for by the evolutionary biology of our species has been the core problem of generative linguistics (see Berwick and Chomsky, 2016, for a recent exposition). Nevertheless, although the origin of human language with its basic unchangeable properties precedes its evolution thenceforth at a conceptual level, their models will necessarily need to support each other. If the unchanging computational basis that sets human languages apart from other systems used semiotically is some sort of merge operation that applies recursively, generating hierarchical structures that affect the semantic interpretation of the combined parts and that are blind to linear order, then language acquisition, processing and also change need to be accounted for solely on this innate basis. Everything else must emerge from its interaction with other domains. The Principles and Parameters framework (Chomsky, 1982; Chomsky \& Lasnik, 1993) understands this crucial direction for validation. Perhaps not surprisingly, it is from within P\&P that explicit dynamical systems start to appear in linguistics (Niyogi \& Berwick, 1995). That said, any theory of linguistic change needs to start from some basic conception of language structure and its nature, with each having different implications for the theory, as illustrated by Labov's critical reflection on the consequences of a strict functionalist view of language evolution (1994, p.9): 
	\vspace{0.5cm}
	\begin{quote}
		If language had evolved in the course of human history as an instrument of communication and been well adapted to that need, one of its most important properties would be stability. No matter how difficult a language was to learn, it would be easier to learn if it were stable than if it continued to change, and no matter how useful a system of communication was, it would be more useful if it could be used to communicate with a neighboring group without learning a new system. The fact of language change is difficult to reconcile with the notion of a system adapted to communication, unless we identify other pathological features inherent in language
		that limit this adaptation.
	\end{quote}     
	\vspace{0.5cm}
	
	Although the research practice of those focusing on the different phenomena of language origin and change may look very different, their proposed solutions to problems ultimately need to be combined into a single theory. It will be maintained in this series of notes, as a general working thesis, that the study of language acquisition is the only bridge capable of connecting those theoretical ends. The general picture underlying this thesis is straightforward, children are initially equipped with the same basic language faculty since its biocognitive origin in the species but the process of language acquisition is also the main mechanism through which linguistic change occurs and propagates at all other levels over different generations \footnote{With the exception of the more conventional aspects of the lexicon, such as the form-meaning association, but also including the argument structure and other morphosyntactic and morphophonological aspects.}.          
		
	In these notes, we work within the acquisition-driven framework for language change put forward by Partha Niyogi and Robert Berwick, in a series of works beginning in the 1990's (Niyogi \& Berwick, 1995, 1997, 2009). The main results, methods, and future research topics of that program came together as a general framework in \textit{The Computational Nature of Language Learning and Evolution} (Niyogi, 2006; henceforth, the NB framework). Our main goal here, which is part of a larger project, is to derive explicit dynamics for a population of learners that follow the Tolerance Principle (Yang, 2016, 2018). This principle has been recently proposed as as a key component of children's cognition in determining the productivity of any linguistic rule (Yang, 2016, 2018). Multiple simulations and applications to case studies will be reserved to future notes. Here, we derive the base dynamics. The NB framework and the Tolerance Principle are presented in sections \ref{NBsection} and \ref{TPsection}, respectively. Section \ref{mainsection} brings the two together for our main analysis.       
	
	\section{An acquisition-driven framework for language change} \label{NBsection}
	
	The NB framework is set out to investigate how slight individual differences in the process of language acquisition can affect the linguistic composition of a population over time. The basic logical premise is that children must infer languages that are slightly different from those of the previous generations, otherwise, "Perfect language acquisition would imply perfect transmission. Children would acquire perfectly the language of their parents, language would be mirrored perfectly in successive generations, and languages would not change with time." (Niyogi, 2006, p.155). In the modern literature, the idea of imperfect learning as a causal force for linguistic change appears at least a few decades earlier in Kiparsky's thesis on \textit{Phonological Change} (1965). The NB framework offers a very general mathematical setting in which the dynamics of the linguistic composition for a particular population of learners can be analysed. The formulation of explicit acquisition-driven dynamical systems for language within this framework not only serves to test the adequacy of different grammars and learning algorithms by examining their evolutionary consequences but also, and conversely, allows the investigation of how much acquistion can account for the propagation and properties of the real dynamics observed in historical data. Social networks, spatial models, and several other factors can also be built from the base models (see Niyogi, 2006, Chapter 10).   
	
	In this section I present a general mathematical description of the NB framework.
	
	Let $C$ denote a community of speakers and $Card(\mathcal{L})$ be the size of a language space $\mathcal{L}$ under consideration. $\mathcal{L}$ is usually the set of languages/grammars considered to be avaliable for a child to hypothesize in a given study. In a Principles and Parameters setting, for example, if one is working with $N$ parameters, then $Card(\mathcal{L})=2^{N}$. In general mathematical terms, the NB framework could be summarized as a map from five major theoretical parameters (that I comment on below) to a discrete dynamical system whose state space is the set of all possible linguistic compositions for $C$ (i.e. the proportion of speakers of each language $L_{i} \in \mathcal{L}$). More formally, the state space is the simplex subset $\{\mathcal{P} \unboldmath \in \mathbb{R}^{Card(\mathcal{L})}: \sum_{i=1}^{Card(\mathcal{L})}\mathcal{P}_{i}=1 \}$, where $\mathcal{P}_{i}$ represents the proportion of $L_{i}$ speakers in the community $C$, with $i=1,...,Card(\mathcal{L})$. Therefore, under certain theoretical hypotheses that take the form of the alluded parameters, the framework allows the derivation of dynamical systems describing how the linguistic composition $\mathcal{P}$ changes over generations of $C$.
	
	To introduce the main concepts involved in the framework, let us first establish a few definitions. We shall understand language as a formal language over some fixed non-empty finite alphabet $A$ and denote by $A^{*}$ the universe of all possible finite strings of elements from $A$. Therefore, the collection of all possible data sets is given by $\mathcal{D}=\bigcup_{n=1}^{\infty}(A^{*})^{n}$. The general parameters are:  \\
	
	\begin{itemize}[noitemsep,nolistsep]
		\item[(a)] The language space $\mathcal{L}$.
		\item[(b)] The learner $\mathcal{A}$, representing the child, which is defined as an effective procedure from $\mathcal{D}$ to $\mathcal{L}$ (technically, a partial recursive function if one assumes the Church-Turing thesis).
		\item[(c)] The notion of convergence to a language, that is, the precise definition of learning being adopted.  
		\item[(d)] The population size, $Card(Pop)$. 
		\item[(e)] The probability distribution $P$ according to which data is drawn and presented to the learners.
	\end{itemize}         
	\vspace{0.5cm}
	Note that each of these concepts is not only very general but also depends on several other subparameters to be well-defined. To give one important example I will return to later, the distribution $P$ depends on the linguistic cohort, that is, the number of different generations assumed to provide the learner with data. Another subparameter is the representation of linguistic elements and the grammar (e.g. optimality-theoretical rankings, rewriting rules) used to specify $\mathcal{L}$ and $\mathcal{A}$. The central procedure to generate an update rule in the state space from these components is to determine a mathematical relation between the average behavior of the individual learner to the composition of the population: 
	
	\begin{equation} \label{mainheuristics}
		\mathcal{P}_{i,t+1} \thicksim \mathbb{P}(\mathcal{A}_{t+1}=L_{i}),
	\end{equation}
	where $\mathcal{P}_{i,t+1}$ denotes the proportion of $L_{i}$ speakers in generation $t+1$ and $\mathbb{P}(\mathcal{A}_{t+1}=L_{i})$ is the probability that a learner from generation $t+1$ will acquire $L_{i}$ (which, in turn, is a value that depends on the composition of the previous generations that provide the learner with input and other factors). The explicitation of such a relation is achieved by working under precise mathematical hypotheses regarding those components, to varying degrees of realism. Niyogi and collaborators, mainly R. Berwick, have explored many important variations of those five parameters. They were able to deduce explicit relations for (\ref{mainheuristics}) and carry out full mathematical analyses in several such cases. Even some of the most idealized models (see Niyogi and Berwick, 1997; Niyogi, 2006, Chapter 5) resulted in relevant dynamical behavior observed in historical linguistics, such as the important $S-$shaped curves (see Labov, 2001; Blythe \& Croft, 2012; Hayes, 2022). Nevertheless, more realistic assumptions on the components, such as finite $Card(Pop)$, often lead to dynamics that are difficult to determine or examine analytically, thus often requiring the use of numerical simulations. 
	
	Niyogi's monograph (2006) not only summarized the main results of this framework but also set up research topics with clear future directions for the investigation of different parameters and possible expansions of the framework. More recently, the Tolerance Principle has been proposed as a component of children's cognitive system governing the productivity of all linguistic rules (Yang, 2016). I propose to investigate the dynamical implications of a population of Tolerance Principle learners (Definition \ref{TPlearners}) as a parametric choice within the NB framework (Section \ref{mainsection}). Before we turn to that, let us first present the main component of our analysis.

	\section{The Tolerance Principle and Language Acquisition} \label{TPsection}
	
	Henceforth, a linguistic rule $R$ will be understood in great generality as an IF-THEN statement that evaluates a decidable condition $A$, which, if met by an object $x$, maps it to $R(x)$. Schematically:
	\[
	R: \text{If} \; A(x), \; \text{then} \; R(x) 
	\]  
	Thus, rules lead to functions that to be well-defined need a representation of the elements that can be evaluated within some formal system, the decidable condition $A$ which specifies the domain for rule application, and the codomain specified by $R$. In this sense, learning a rule pressuposes a representation/structural description of the elements evaluated by $A$, the specification of $A$ within that formal system, and the association to be made with each of the elements that satisfy that condition. Therefore, at this level of generality, rules are presented without commiting to a particular theory of what the representations of linguistic elements should be. This serves as a framework for any particular theory of linguistic rules. 
	
	However innocuos and acceptable a description may seem, it always involves a particular abstract model, so it is important to remark that, unless otherwise stated, any examples of rules I provide in this series of notes are given informally for quick illustration, and not in defense of a particular theoretical proposal of, say, (morpho)phonology or (morpho)syntax. For instance, the phonemic symbols below can be understood as place holders for feature matrices or some other structural representation from different phonological theories (feature geometry, gestures and so forth). I leave it unspecified. In future notes we will evaluate specific theoretical proposals and their implications for acquisition and change, but this will be made clear then. That said, let us look at some informal examples for morphophonological rules involving plural forms: \\
	
	\begin{itemize}
		\item[a.] (Plural for English \textit{cat}):
		If $x=\text{\textipa{k\ae t}}$, then add $\text{\textipa{s}}$ to the end of the string. 
		\item[b.] (Plural for English \textit{dog}):
		If $x=\text{\textipa{d\textturnscripta g}}$, then add $\text{\textipa{z}}$ to the end of the string. 
		\item[c.] (Plural for English \textit{cactus}):
		If $x=\text{\textipa{k\ae kt\textschwa s}}$, then remove $\text{\textipa{\textschwa s}}$ and replace with $\text{\textipa{ai}}$. 
		\item[d.] (Plural for English \textit{-oo-} nouns):
		If $x=\alpha \text{\textipa{\textupsilon}} \beta$ or $x=\alpha \text{\textipa{u:}} \beta$, then $\alpha \text{\textipa{i:}} \beta$
		\item[e.] (Plural for English, Catalan, Spanish, and Portuguese nouns):
		If $x$ is a noun, then add $s$ to the end of the string. \\
	\end{itemize} 
	
	There is vast empirical evidence that in the process of language acquisition, children induce and apply certain rules to overgeneralizaed domains that go beyond the data provided by adults from their linguistic communities, whereas some other rules (almost) never apply to forms for which children have no positive evidence. The former type of rule is called \textit{productive}, as opposed to the latter. Furthermore, children are conservative even in the overapplication of productive rules, matching closely the positive data (see Yang 2016 and references therein for a variety of crosslinguistic evidence).   
	
	So what makes a rule productive? In other words, how do children sort out productive rules from unproductive ones? 
	The Tolerance Principle (Yang 2016), henceforth TP, is a proposed solution to this long-standing problem. Not only that, this is a precise quantitative proposal with measurable effects, arguably the first candidate to a quantitative law in linguistic theory. The final form for the TP is presented as follows (Yang, 2016, p.64):
		
	\begin{equation} \label{TP}
	\hspace*{-0.1cm} 
	\begin{split} 
	& \textit{Tolerance Principle} \\ 
	&	\text{Let} \; R \; \text{be a rule applicable to} \; N \; \text{itens, of which} \; e \; \text{are exceptions.} \\ &\text{Then,} \; R \; \text{is productive if and only if}\\ & \hspace{3.5cm}  e \leq \theta_{N}:=N/\ln N.  
	\end{split}  		 
	\end{equation}
	
	An \textit{exception} to a rule $R$ is an item that satisfies the condition for the application of $R$, but the rule does not actually apply to that item in the language (e.g. rule $d.$ above does not apply to \textit{book} nor \textit{cook}). Rules such as $a.$, $b.$ and $c.$, that refer to a single element and hold no exceptions have an important role in the derivation of the TP. We will refer to them as \textit{listings}.
	
	The heart of the TP derivation in its most basic form, which is not the asymptotic version presented in (\ref{TP}), is the Elsewhere Condition Model for language processing (Yang, 2016, p.49 and references therein). For a rule $R$, the Elsewhere Condition Model establishes a serial search procedure in which exception listings for $R$ are listed prior to $R$ and are ordered by frequency.
	
	\begin{definition} \label{ECA}(Elsewhere Condition Model)
	Let $R$ be a rule applicable to $N$ itens in a data set $D$ for a language $L$. Let $\{e_{1}, ...,e_{m}\}$ be exceptions to $R$ listed by their order of frequency and $\{r_{1},...,r_{n}\}$ be the itens to which $R$ actually applies in $L$, with $m+n=N$. Then, the \textit{Elsewhere Condition Model} for $R$ is defined as:\\
	
	\noindent If $x=e_{1}$, then... \\
	\vdots \\
	If $x=e_{m}$, then... \\
	$R$. \\
	\end{definition} 
	
	\begin{remark} (Ranked Listing Model) Any Elsewhere Condition algorithm is weakly equivalent to serial listing. Indeed, using the same notation from definition \ref{ECA}, one can obtain the same input-output pairs by listing the association for each item individually: \\
		
	\noindent If $x=e_{1}$, then... \\
	\vdots \\
	If $x=e_{m}$, then... \\
	If $x=r_{1}$, then $R(x)$. \\
	\vdots \\
	If $x=r_{n}$, then $R(x)$. \\
	
	We can then \textbf{rearrange} this list following the order of frequency of all $N$ itens (with a possible mixture of the $e_{i}$ and $r_{j}$). Let us refer to this final listing rearranged by frequency as a \textit{Ranked Listing Model}. 
	\end{remark}

	Let $T(N,e)$ and $T(N,N)$ denote the expected time complexity of the Elsewhere Condition Model and its associated Ranked Listing Model, respectively. In its purest form, the Tolerance Principle is a conjecture that children opt for the model with the smaller expected time of access to a form (Yang, 2016, p.61):
	
	\begin{equation} \label{TPbaseform}
	\hspace*{-0.1cm} 
	\begin{split} 
		& \textit{Tolerance Principle (Base Form)} \\ 
		&	\text{A rule} \; R \; \text{is productive if} \; T(N,e)<T(N,N) \; \text{; otherwise} \; R \; \text{is unproductive} \\
	\end{split} 
	\end{equation}

	To derive the closed form in (\ref{TP}) from (\ref{TPbaseform}), two mathematical assumptions are held. The first, used to calculate $T(N,e)$ and $T(N,N)$ explicitly, is that the $N$ itens ranked by frequency follow the empirical Zipf's law (Zipf 1936, 1949; ), which states that the frequency $f_{i}$ of the $i$-th element holds an inverse relation with its rank:
	\[
	f_{i}= \frac{1/i^{s}}{\sum_{k=1}^{N}1/k^{s}}
	\]   
	Yang (2016, p.62) holds $s=1$, that is, a proper inverse proportion between rank and frequency, thus yielding $T(N,N)=\frac{N}{H_{N}}$ and $T(N,e)=\frac{e}{N}\frac{e}{H_{e}}+(1-\frac{e}{N})e$, where $H_{j}$ is the $j$-th Harmonic Number $\sum_{k=1}^{j}\frac{1}{k}$. Therefore, under Zipf's Law with exponent $1$, the base Tolerance Principle implies that
	\begin{equation} \label{TPintermediateform}
		\hspace*{-0.1cm} 
		\begin{split} 
			& \textit{Tolerance Principle (Intermediate Form - under Zipf's law)} \\ 
			&	\text{A rule} \; R \; \text{is productive if} \; \frac{e}{N}\frac{e}{H_{e}}+(1-\frac{e}{N})e<\frac{N}{H_{N}} \; \text{; otherwise} \; R \; \text{is unproductive} \\
		\end{split} 
	\end{equation}
	
	The second mathematical assumption used to produce an analytical solution is to work with large $N$. Following this strategy with Sam Gutmann, Yang (2016, p.63) replaces the Harmonic Number by the asymptotically equivalent natural logarithm (i.e. $\lim_{N \to \infty} \frac{H_{N}}{ln N}=1$) in $T(N,e)=T(N,N)$ and solves for an expression that makes the equality hold in the limit as $N$ increases. This led to the final form in (\ref{TP}), which theoretically speaking, should only work as an approximate threshold of productivity for large lists of itens (large $N$). However, experimental studies with aritifical languages suggest that $N/ln N$ is a categorical threshold even for small $N$ (e.g. $N=9$) (see Schuler, Yang, \& Newport 2016, 2021). As a mathematical side note, it is interesting to observe that $N/ln N$ is asymptotically equivalent to the number of primes smaller or equal to $N$, as stated by the Prime Number Theorem.   
	
	The reader is referred to Yang (2016) for the complete argument, empirical motivation, and for the analysis of a wide range of acquisition data that receive a unified account under the Tolerance Principle (henceforth TP). In these notes, we mainly pursue the implications of the TP, as a proposed component of children's cognition, for language change. This direction of research is also laid out by Yang later in his book (2016, p.157) (see also Ringe and Yang 2022 for a recent application): 
	\vspace{0.5cm}
	\begin{quote}   	
	Proposed as a solution for the productivity problem, the Tolerance Principle	must be a component of the human cognitive system that governs all linguistic matters regardless of time or place. One approach to validation, which I have pursued in the preceding pages, is to examine a wide range of crosslinguistic phenomena and the processes by which children acquire them. In those cases, the target state of the linguistic systems is fairly well understood, which forms	a stable point of comparison against the predictions of the Tolerance Principle. Yet another, and possibly more interesting, application can be found in his-
	torical linguistics. I contend that the Tolerance Principle is a causal force that shapes the history of languages. As such, it has the potential of uncovering the deterministic factors in language change. 
	\end{quote} 
	\vspace{0.5cm}
	If the analytical basis of the NB framework is true, that is, language acquisition is the main mechanism through which linguistic change occurs and propagates, and if the TP is indeed a key cognitive component of acquisition, then Yang's thesis that the TP is a causal force that shapes the history of languages is an inevitable conclusion. Motivated by these two premisses, we will examine the dynamical implications of the TP using the techniques from the NB framework. As far as I know, the iterated maps derived in the next section are the first explicit dynamical systems of language based on the TP.

	\section{The Dynamics of Language for TP Learners} \label{mainsection}
	
	Let $R$ be a rule defined as in the previous section. Denote by $D_{N}^{R, L}$ a data set from language $L$ consisting of $N$ pairs of the form $(x, S(x))$, where each $x$ falls into the structural category evaluated by $R$, that is, the IF part of $R$ is satisfied by $x$. If $S(x)=R(x)$, the data point $(x,S(x))$ is evidence of $R$ application; if $S(x)\neq R(x)$, then $(x,S(x))$ is an exception to $R$ in $L$. Henceforth, we leave $L$ implicit to avoid clutter and denote the data sets by $D_{N}^{R}$. Given a data set $D_{N}^{R}$, the task of a learner $\mathcal{A}$ will be to decide between $R^{+}$ (holding the rule productive) or $R^{-}$ (dispensing with the rule). Then we can naturally define Tolerance Principle learners:
	
	\begin{definition} \label{TPlearners} (TP Learners) An algorithm $\mathcal{A}$ is a TP learner with respect to a rule $R$ if it is defined as follows:
		
	\begin{align*}
	\mathcal{A}(D_{N})=R^{+}, \; &\text{if} \; e \leq N/\ln N  \\
	\mathcal{A}(D_{N})=R^{-}, \; &\text{otherwise}. 
	\end{align*}

	\end{definition}
	
	Denote by $\alpha_{t}$ the proportion of $R^{+}$ speakers in generation $t$. Thus, $1-\alpha_{t}$ is be the proportion of speakers in generation $t$ that dispensed with $R$ (adopting some other rule). Let also $p_{t,R^{+}}(e)$ stand for the probability with which an $R^{+}$ speaker from generation $t$ produces an exception to $R$ (one of the exceptions listed prior to the rule in the serial Elsewhere Condition Model). Similarly, $p_{t,R^{+}}(R)$ will denote the probability with which an $R^{+}$ speaker from generation $t$ produces an output of $R$. Hence, $p_{t,R^{+}}(e)=1-p_{t,R^{+}}(R)$. Analagously, $p_{t,R^{-}}(e)$ and $p_{t,R^{-}}(R)$ define these probabilities for $R^{-}$ speakers. As a starting point, we derive an update rule $\alpha_{t}$ under the following simplifying assumptions: \\
	
	\begin{itemize}[noitemsep,nolistsep]
		\item[$(i)$] The population is infinite. 
		\item[$(ii)$] The data points $d_{i} \in D_{N}^{R}$ are drawn in $i.i.d.$ fashion from the environment. 
		\item[$(iii)$] The probabilities $p_{t, R^{+}}(e)$ and $p_{t,R^{-}}(e)$ are constant through all generations. Since there is no dependence on $t$, henceforth we simply denote $p_{t, R^{+}}(e):=p_{R^{+}}(e)$ and $p_{t,R^{-}}(e):=p_{R^{-}}(e)$.   
		\item[$(iv)$] Every learner $\mathcal{A}_{t+1}$ receives data from $R^{+}$ and $R^{-}$ speakers from the previous generation in the proportion $\alpha_{t}$ and $1-\alpha_{t}$, respectively.
		\item[$(v)$] Every learner receives data from the previous generation only. 
	\end{itemize}   
	\vspace{0.5cm}
	\begin{theorem} \label{basicTPdynamics} Under assumptions $(i)$ to $(v)$, the dynamics of TP learners is given by
	\[
	\alpha_{t+1}=\sum_{k=0}^{\lfloor N/\ln N \rfloor} \binom{N}{k}[\alpha_{t}p_{R^{+}}(e)+(1-\alpha_{t})p_{R^{-}}(e)]^{k}[\alpha_{t}p_{R^{+}}(R)+(1-\alpha_{t})p_{R^{-}}(R)]^{N-k}.
	\] 	
	\end{theorem}

\begin{proof}
	Let $\mathbb{P}(\mathcal{A}(D_{N}^{R})=R^{+})$ denote the probability with which a learner $\mathcal{A}$ converges to $R^{+}$. Since the population is infinite $(i)$ and the data for every learner is drawn from the same distribution $(iv)$, we have that $\alpha_{t+1}=\mathbb{P}(\mathcal{A}_{t+1}(D_{N}^{R})=R^{+})$. Therefore, we only need to compute this probability. We do this using basic combinatorics under assumptions $(ii)$ to $(v)$. Indeed, by $(iv)$ and $(v)$, we know that the probability of an exception being presented to a learner from generation $t+1$ is 
	\[
	\alpha_{t}p_{R^{+}}(e)+(1-\alpha_{t})p_{R^{-}}(e).
	\]
	Thus, the probability of a particular data set $D_{N}^{R}$ with $k$ exceptions that is drawn $i.i.d$ ($ii$) is
	\[
	[\alpha_{t}p_{R^{+}}(e)+(1-\alpha_{t})p_{R^{-}}(e)]^{k}[\alpha_{t}p_{R^{+}}(R)+(1-\alpha_{t})p_{R^{-}}(R)]^{N-k}.
	\]
	Since $k$ exceptions can show up in $\binom{N}{k}$ different ways in a sequence of $N$ data points, we multiply the previous term by the binomial coefficient to obtain the probability for all $D_{N}^{R}$ with exactly $k$ exceptions. Finally, we sum over the number of exceptions from $k=0$ to $k= \lfloor N/\ln N \rfloor$ to account for all the possibilities that make $\mathcal{A}$ converge to $R^{+}$, giving us the expression in the theorem. Note that $(iii)$ assures us that $\alpha_{t}$ is the only variable, thus guaranteeing the dynamics.  
\end{proof}

\begin{remark} \label{variantfrequency1} (Relation to Variant Frequency - Comparison to NB Models)
	Suppose $R$ application denotes one variant $V_{1}$ and the exceptions represent a second variant $V_{2}$. Therefore, the proportion $\alpha_{t}$ of $R^{+}$ speakers from generation $t$, to which the dynamics refer in our notation, do not refer directly to the frequency with which $V_{1}$ is applied. Rather, this frequency is given by 
	\begin{equation} \label{ruleproductivityandvariantfrequency}
		\alpha_{V_{1},t}=\alpha_{t}p_{R^{+}}(R)+(1-\alpha_{t})p_{R^{-}}(R),
	\end{equation}
	with the evolution of $\alpha_{t}$ provided by the update map. This is essential to consider when comparing simulations with historical data for all dynamics derived in these notes. In the NB models, the notation $\alpha_{t}$ often embodies the frequencies directly. 
\end{remark}
	
	Although less conceptually explicit, note that the dynamics can also be expressed in a more computation-friendly way as:
	\begin{equation} \label{basicTPdynamics2}
	\alpha_{t+1}=\sum_{k=0}^{\lfloor N/\ln N \rfloor} \binom{N}{k}[\alpha_{t}p_{R^{+}}(e)+(1-\alpha_{t})p_{R^{-}}(e)]^{k}[1-(\alpha_{t}p_{R^{+}}(e)+(1-\alpha_{t})p_{R^{-}}(e))]^{N-k}.
	\end{equation}
	
	Fixing $N$ for the size of the data set that will decide the convergence of a learner, we see from these expressions  that there are two parameters in this model: $p_{t}^{R^{+}}(e)$ and $p_{t}^{R^{-}}(e)$. Some constraints are expected to hold empirically: \\
	
	\begin{itemize}[noitemsep,nolistsep]
		\item[(a)]  $p_{R^{-}}(e)>p_{R^{+}}(e)$, which is equivalent to $p_{R^{+}}(R)>p_{R^{-}}(R)$.
		\item[(b)] $p_{R^{+}}(R)>p_{R^{+}}(e)$. 
		\item[(c)] $p_{R^{-}}(e)>p_{R^{-}}(R)$
	\end{itemize}   
	\vspace{0.5cm}
	However, at this general abstract level, none of these relations should be assumed to always hold, especially considering the type vs. token decision to be considered for data sets and also their size $N$ in the context of language acquisition. Moreover, one path for change from a majorly rule-preserving population to a population that dispenses with a previously productive rule may be open precisely when one or more of these relations do not hold. Nevertheless, we show a simulation below in which such a transition occurs while all the above constraints hold as well as $p_{R^{+}}(R)>p_{R^{-}}(e)$.   
	
	\begin{figure}[htbp]
		\centering
		\includegraphics[scale=.5]{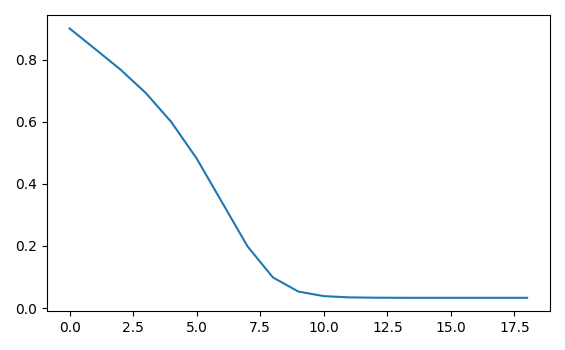}
		\caption{$N=9$, $\alpha_{t}=0.9$, $p_{R^{+}}(e)=0.2$, $p_{R^{-}}(e)=0.7$}
		\label{fig1}
	\end{figure}	

	We see that for learners that converge for small data sets of size $9$ and parameters $p_{t}^{R^{+}}(e)=0.2$ and $p_{t}^{R^{-}}(e)=0.7$, an initial population in which $90\%$ of speakers use $R$ productively transitions in about $10$ generations to a stable fixed point where only a small fraction of  speakers holds $R$ productive. The curve follows the well-known $S-$shaped pattern from historical linguistics, in which change is initially slow, then accelerates for some generations before changing slowly again towards a regular state. 	
	
	We initiate a general qualitative analysis of the dynamics from Theorem (\ref{basicTPdynamics}) by considering homogeneous populations. 
	
	\textbf{Homogeneous populations:} Let us consider two general cases in which the population consists entirely of productive $R^{+}$ speakers ($\alpha=1$) or entirely of $R^{-}$ speakers ($\alpha=0$). \\
	
	If all speakers in the population are $R^{+}$, then 
	\[
	f(1)=\sum_{k=0}^{\lfloor N/\ln N \rfloor} \binom{N}{k}[ p_{R^{+}}]^{k}[1-p_{R^{+}}]^{N-k}.
	\]
	
	Note that $f(1)=1$ if and only if $p_{R^{+}}=0$ (to see the first implication, differentiate $f(1)$ with respect to $p_{R^{+}}$), so the only chance for the population to remain homogeneous with respect to the productivity of a rule $R$ is if the probability of producing an exception is zero. Therefore, any exception in the environment leads to mixed populations. Also note that $f(1)=0$ if and only if $p_{R^{+}}=1$, that is, if $R^{+}$ speakers never apply the rule, which is completely unrealistic as a choice of parameter. Therefore, there is no chance of going from one homegeneous generation of $R^{+}$ speakers to a homogenous population of $R^{-}$ in a single generation. Intermediate mixed populations are predicted by the dynamics. The analysis for $\alpha=0$ is symmetric with respect to the parameter $p_{R^{-}}$. Precisely, $f(0)=0$ iff  $p_{R^{-}}=1$ and $f(0)=1$ iff $p_{R^{-}}=0$ (which, again, is not a relevant parametric choice). To summarize the main points: \\
	
	\begin{itemize}[noitemsep,nolistsep]
		\item[-] A homogeneous population of $R^{+}$ ($R^{-}$) speakers will necessarily become mixed unless no exception (rule application) is found in the data. 
		
		\item[-] A homogenous population of $R^{+}$ ($R^{-}$) speakers cannot shift entirely to a homogeneous population of $R^{-}$ ($R^{+}$) speakers in a single generation. \\ 
	\end{itemize}   
	
	The first item predicts that unless a rule $R$ has absolutely no exceptions, as theoretically expected from, say, the evaluation of a UG principle, then change is inevitable as the productive status of a rule becomes mixed in the population. Since most rules, even productive ones, do have exceptions, this offers a formal explanatory route  for Labov's observation, cited in the beginning of these notes, that most things in language change except for a bare skeleton of unchangeable abstract relations (e.g. capacity for recursion). The second property predicts that an entire population that holds a rule productive cannot change into a homogeneous population of speakers that hold the same rule unproductive at the turn of a single generation. This should be expected as a necessary condition for any candidate model of linguistic change. Although this property alone is not sufficient for stating that a model has explanatory power over the actual dynamics of language, it is certainly necessary. These dynamics differ from the general Cue-Based Learners analysed by Niyogi (2006, p.173) which cannot generate spontaneous change from a homogeneous composition of learners since $\alpha=0$ and $\alpha=1$ are both stable fixed points in those models.  
	
	%\textbf{Symmetric behaviour:} What happens if the probability of $R^{+}$ speakers producing exceptions is the same as the probability of $R^{-}$ speakers applying $R$? In other words, what the dynamics predicts for $p_{R^{+}}=1-p_{R^{-}}$. 

	\subsection{Stochastic dynamics (finite populations)}
	
	Let us consider now the case of a finite population. We replace assumption $(ii)$ from the previous section by: \\
	\begin{itemize}[noitemsep,nolistsep]
		\item[$(ii)^{*}$] The population is finite, consisting of $S$ adult speakers. 
	\end{itemize} 
	\vspace{0.5cm}
	We follow Niyogi's presentation (2006, p.306, section 10.1.1) with the particular choice of TP learners as our algorithm here. Let us represent the linguistic configuration of a generation $t$ as a vector $(X_{1}(t),...,X_{S}(t))$ with $X_{i}(t)\in\{0,1\}$. Here, $X_{i}(t)=1$ will mean that the speaker $i$ from generation $t$ is $R^{+}$ and $X_{i}(t)=0$ means that $i-$th speaker from gen $t$ is $R^{-}$. Thus, the proportion of $R^{+}$ speakers from gen $t$ can be written as
	\begin{equation} \label{averagelinguisticbehaviour}
		Y(t)=\frac{1}{N}\sum_{i=1}^{N}X_{i}(t).
	\end{equation}   

	So, $Y(t) \in \{\frac{0}{S}, \frac{1}{S},...,\frac{S}{S}\}$. Note that $Y(t)$ is just a new notation for $\alpha_{t}$ from the previous section. Each $X_{i}(t+1)$ is a random variable. Assuming $(ii)^{*}$ together with the other four hypotheses as before, we already know from the proof of Theorem \ref{basicTPdynamics} that
	
	\begin{equation} \label{probconvergencetoR+}
	P(X_{i}(t+1)=1)=f(Y(t),N):=\sum_{k=0}^{\lfloor N/\ln N \rfloor} \binom{N}{k}[P(Y(t))]^{k}[1-P(Y(t))]^{N-k},
	\end{equation}
	where $P(Y(t))=Y(t)p_{R^{+}}(e)+(1-Y(t))p_{R^{-}}(e)$. 
	\vspace{0.5cm}
	
	Therefore, the probability for the average linguistic behaviour $Y(t+1)$ to take a certain value is solely determined by $Y(t)$.
	
	As usual for finite population analysis, in what follows, we also assume the following simplifying condition which is implicit in Niyogi's analysis (Niyogi, 2006, p.307): \\
	\begin{itemize}[noitemsep,nolistsep]
		\item[$(ii)^{*}(b)$] The size $S$ of the population remains constant across generations.  
	\end{itemize}
	\vspace{0.5cm}
	Although this hypothesis is extremely unrealistic with respect to global populations/general language communities, it can serve as approximation for certain bounded physical spaces, such as particular neighbourhoods, for a number of generations. Spatial models will be considered in future notes.
	
	Finally, we turn to the evolution of the average linguistic composition. Let $T_{ij}$ represent the probability with which the population transitions from state $Y(t)=\frac{i-1}{S}$ to $Y(t+1)=\frac{j-1}{S}$, for any $1 \leq j, i \leq S+1$. This defines the transition matrix for the evolution of the average linguistic behaviour, which can be viewed as a Markov chain with $S+1$ states. Hence, 
	\begin{proposition}
		Under assumptions $(ii)^{*}$, $(ii)^{*}(b)$, and $(i), (iii), (iv), (v)$, the transition matrix for the average linguistic behaviour of a finite population of TP learners is given by
		\[
		T_{ij}=\binom{S}{j-1}f(\frac{i-1}{S},N)^{j-1}\bigg(1-f(\frac{i-1}{S},N)\bigg)^{S-(j-1)},
		\]
	with $f$ defined as in equation $(\ref{probconvergencetoR+})$.
	\end{proposition}

	\subsection{Learning from three generations}
	
	Next, we make a more realistic assumption on the number of previous generations that provide learners with data. If the current population of learners/children is $t+1$, then the population of living adults will consist of a proportion of speakers from generation $t$, another proportion of adults from generation $t-1$, and so on. We will refer to the proportion of living adults from generation $s$ in a population as $Pop(s)$. 
	
	Now we make the following alternative assumption to hypothesis $(v)$ and add a symplifying condition related to this alternative:  \\
	
	\begin{itemize}[noitemsep,nolistsep]
		\item[$(v)^{*}$] Every learner receives data from three previous generations.
		\item[$(v)^{*}(b)$] For each new generation, the proportion of speakers from each of the three previous generations in the population does not change. More precisely, for any generation $t+1$, $Pop(t-2)=A$, $Pop(t-1)=B$,  and $Pop(t)=C$, for constants $A$, $B$ and $C$, with $A+B+C=1$.   
	\end{itemize}   
	\vspace{0.5cm}

	For a generation of children $t+1$, we can think of $t-2$, $t-1$ and $t$ as representing the generations of the grandparents, parents, and older siblings, respectively. Thus, $(v)*(b)$ states that the proportion of these groups does not change for any $t$. For example, the percentage of grandparents with respect to the other two generations will always be a constant $A$. The arguments of this section can be trivially adapted for any number $M$ of previous generations providing learners with data, although going beyond, say, $M=4$, has little empirical significance.       
	
	Under these new assumptions, we derive the appropriate evolutionary dynamics for TP learners that receive data from three previous generations.  

	\begin{proposition} \label{3genTPdynamics} Under assumptions $(i)$ to $(iv)$, $(v)*$ and $(v)*(b)$, the dynamics of TP learners is given by
			\[
			\alpha_{t+1}=\sum_{k=0}^{\lfloor N/\ln N \rfloor} \binom{N}{k}[Ag_{t-2}^{e}+Bg_{t-1}^{e}+Cg_{t}^{e}]^{k}[Ag_{t-2}^{R}+Bg_{t-1}^{R}+Cg_{t}^{R}]^{N-k},
			\] 	
			where $g_{s}^{x}:=\alpha_{s}p_{R^{+}}(x)+(1-\alpha_{s})p_{R^{-}}(x)$. 
	\end{proposition}
	
	\begin{proof} To see this is true, we only need to consider the same argument from Theorem \ref{basicTPdynamics} and then take into account that an exception may now be provided by speakers from generations $A$, $B$ or $C$. 
	\end{proof}
	
	In terms of the canonical inner product and the probability complement, the dynamics can also be written as
	\begin{equation} \label{TPdynamics3gen}
		\alpha_{t+1}=\sum_{k=0}^{\lfloor N/\ln N \rfloor} \binom{N}{k} [(A,B,C)\cdot (g_{t-2}^{e}, g_{t-1}^{e}, g_{t}^{e})]^{k}[1-(A,B,C)\cdot (g_{t-2}^{e}, g_{t-1}^{e}, g_{t}^{e})]^{N-k}.
	\end{equation}
	
	\appendix
	\section*{Appendix}
	\addcontentsline{toc}{section}{Appendices}
	\renewcommand{\thesubsection}{\Alph{subsection}}
	\subsection{First derivative for stability analysis of the base dynamics from Theorem \ref{basicTPdynamics}.}
	
	For this appendix, we let
	\begin{equation} \label{basedynamicsfunction}
		f(\alpha)=\sum_{k=0}^{\lfloor N/\ln N \rfloor} \binom{N}{k}[\alpha p_{R^{+}}+(1-\alpha)p_{R^{-}}]^{k}[1-(\alpha p_{R^{+}}+(1-\alpha)p_{R^{-}})]^{N-k}.
	\end{equation}
	
	Then,
	\begin{align*}
	f'(\alpha)=&\sum_{k=0}^{\lfloor N/\ln N \rfloor} \binom{N}{k} k [\alpha p_{R^{+}}+(1-\alpha)p_{R^{-}}]^{k-1}(p_{R^{+}}-p_{R^{-}})[1-(\alpha p_{R^{+}}+(1-\alpha)p_{R^{-}})]^{N-k} + \\
	&\binom{N}{k}[\alpha p_{R^{+}}+(1-\alpha)p_{R^{-}}]^{k} (N-k)[1-(\alpha p_{R^{+}}+(1-\alpha)p_{R^{-}})]^{N-k-1}(p_{R^{-}}-p_{R^{+}}).
	\end{align*}
	
	Expanding the binomial coefficient and factoring out some common terms, we have:
	\begin{align*}
		f'(\alpha)&=\sum_{k=0}^{\lfloor N/\ln N \rfloor} \frac{N!}{(N-k)!k!} k [\alpha p_{R^{+}}+(1-\alpha)p_{R^{-}}]^{k-1}(p_{R^{+}}-p_{R^{-}})[1-(\alpha p_{R^{+}}+(1-\alpha)p_{R^{-}})]^{N-k} - \\
		&\frac{N!}{(N-k)!k!}[\alpha p_{R^{+}}+(1-\alpha)p_{R^{-}}]^{k} (N-k)[1-(\alpha p_{R^{+}}+(1-\alpha)p_{R^{-}})]^{N-k-1}(p_{R^{+}}-p_{R^{-}}) \\
		&=N(p_{R^{+}}-p_{R^{-}}) \bigg[\sum_{k=1}^{\lfloor N/\ln N \rfloor} \frac{(N-1)!}{(N-k)!(k-1)!}  [\alpha p_{R^{+}}+(1-\alpha)p_{R^{-}}]^{k-1}[1-(\alpha p_{R^{+}}+(1-\alpha)p_{R^{-}})]^{N-k} \bigg] - \\
		&N(p_{R^{+}}-p_{R^{-}}) \sum_{k=0}^{\lfloor N/\ln N \rfloor}\bigg[ \frac{(N-1)!}{(N-k-1)!k!}[\alpha p_{R^{+}}+(1-\alpha)p_{R^{-}}]^{k} [1-(\alpha p_{R^{+}}+(1-\alpha)p_{R^{-}})]^{N-k-1} \bigg]. 
	\end{align*}

	\vspace{0.5cm}
	Let $P(\alpha):=\alpha p_{R^{+}}+(1-\alpha)p_{R^{-}}$, so we rewrite it more compactly as:
	\begin{align*}
		f'(\alpha)&=N(p_{R^{+}}-p_{R^{-}}) \bigg[\sum_{k=1}^{\lfloor N/\ln N \rfloor} \frac{(N-1)!}{(N-k)!(k-1)!}  P(\alpha)^{k-1}[1-P(\alpha)]^{N-k} \bigg] - \\
		&N(p_{R^{+}}-p_{R^{-}}) \bigg[ \sum_{k=0}^{\lfloor N/\ln N \rfloor} \frac{(N-1)!}{(N-k-1)!k!}P(\alpha)^{k} [1-P(\alpha)]^{N-k-1} \bigg]. 
	\end{align*}

	\vspace{0.5cm}
	Cancelling terms:
	\begin{align*}
		f'(\alpha)&=-N(p_{R^{+}}-p_{R^{-}})\frac{(N-1)!}{(N-\lfloor N/\ln N \rfloor-1)!\lfloor N/\ln N \rfloor!}P(\alpha)^{\lfloor N/\ln N \rfloor}[1-P(\alpha)]^{N-\lfloor N/\ln N \rfloor -1}.
	\end{align*}
	
	Hence,
	\begin{equation} \label{derivativebasedynamics}
		f'(\alpha)=(p_{R^{-}}-p_{R^{+}})N\binom{N-1}{\lfloor N/\ln N \rfloor}P(\alpha)^{\lfloor N/\ln N \rfloor}(1-P(\alpha))^{N-\lfloor N/\ln N \rfloor -1}.
	\end{equation}
	
	\vspace{0.5cm}
	
	The function from $(\ref{basedynamicsfunction})$ and its derivative in $(\ref{derivativebasedynamics})$ will be used for the qualitative analysis and simulations of the base dynamics from Theorem $(\ref{basicTPdynamics})$.

\end{document}